\documentclass[10pt, twoside]{article}
\usepackage[margin=1in]{geometry}

\pdfoutput=1

\usepackage[utf8]{inputenc} 
\usepackage[T1]{fontenc}    
\usepackage{hyperref}       
\usepackage{url}            
\usepackage{booktabs}       
\usepackage{amsfonts}       
\usepackage{nicefrac}       
\usepackage{microtype}      

\usepackage{microtype}
\usepackage{graphicx}
\usepackage{booktabs} 

\usepackage{rotating}
\usepackage{tabularx}

\usepackage{url}            
\usepackage{subfig}

\usepackage{amsfonts,amsmath,amssymb,amsthm,bm,bbm}

\usepackage{algorithm,algpseudocode}

\usepackage{dsfont,array}

\usepackage{appendix}

\def\*#1{\mathbf{#1}}
\def\_#1{\mathcal{#1}}
\def\-#1{\mathbb{#1}}

\usepackage{todonotes}

\usepackage{multirow}
\usepackage{multicol}
\usepackage{xspace}
\usepackage{mathtools}
\usepackage{placeins}
\usepackage{paralist}

\newtheorem{proposition}{Proposition}
\newtheorem{remark}{Remark}

\DeclarePairedDelimiterX\Basics[1](){ #1}

\newcommand{\norm}[1]{\left\lVert#1\right\rVert}

\newcounter{const-no}

\DeclareMathOperator*{\argmin}{\arg\!\min}

\def\elec{{\sf electricity}\xspace}
\def\traffic{{\sf traffic}\xspace}
\def\wiki{{\sf wiki}\xspace}
\def\pems{{\sf PeMSD7(M)}\xspace}

\def\Ytr{\*Y^{\text{(tr)}}\xspace}
\def\Yte{\*Y^{\text{(te)}}\xspace}
\def\predYte{\hat{\*Y}^{\text{(te)}}\xspace}
\def\Xtr{\*X^{\text{(tr)}}\xspace}
\def\Xte{\*X^{\text{(te)}}\xspace}
\def\predXte{\hat{\*X}^{\text{(te)}}\xspace}

\usepackage{titlesec}
\titlespacing\section{0pt}{\parskip - 0.4em}{\parskip - 0.3em}
\titlespacing\subsection{0pt}{\parskip - 0.4em}{\parskip - 0.5em}
\titlespacing\paragraph{0pt}{\parskip - 0.5em}{0pt}

\newcommand{\rbr}[1]{\left({#1}\right)}

\newcommand{\cbr}[1]{\left\{{#1}\right\}}

\newcommand{\cT}{\mathcal{T}}
\newcommand{\cJ}{\mathcal{J}}
\newcommand{\cI}{\mathcal{I}}
\newcommand{\cL}{\mathcal{L}}
\newcommand{\cR}{\mathcal{R}}

\usepackage{hyperref}
\usepackage{multirow}
\usepackage{authblk}

\newcommand{\dln}{{\sf LeveledInit}\xspace}
\newcommand{\dlnprev}{{\sf DLN} \xspace}
\newcommand{\golf}{{\sf DeepGLO}\xspace}
\newcommand{\MFLN}{{\sf TCN-MF}\xspace}
\newcommand{\trmf}{{\sf TRMF}\xspace}
\newcommand{\deepar}{{\sf DeepAR}\xspace}

\graphicspath{{figs/}}

\date{\today}
\title{Think Globally, Act Locally: A Deep Neural Network Approach to High-Dimensional Time Series Forecasting}

\author[1]{Rajat Sen}
\author[1]{Hsiang-Fu Yu}
\author[2]{Inderjit Dhillon}
\affil[1]{Amazon}
\affil[2]{Amazon and UT Austin}

\usepackage[parfill]{parskip}

\begin{document}



\maketitle





\begin{abstract}
Forecasting high-dimensional time series plays a crucial role in many
applications such as demand forecasting and financial predictions. Modern
datasets can have millions of correlated time-series that evolve
together, i.e they are extremely high dimensional (one dimension for each
individual time-series). There is a need for exploiting global patterns
and coupling them with local calibration for better prediction.
However, most recent deep learning approaches in the literature are
one-dimensional, i.e, even though they are trained on the whole dataset, during
prediction, the future forecast for a single dimension mainly depends on
past values from the same dimension.
In this paper, we seek to correct this deficiency and propose
\golf, a deep forecasting model which \emph{thinks globally and acts locally}.
In particular, \golf is a hybrid model that combines a \emph{global} matrix factorization model 
regularized by a temporal convolution network, along with another temporal network that can capture \emph{local} properties 
of each time-series and associated covariates. Our model can be trained effectively on high-dimensional but
diverse time series, where different time series can have vastly different
scales, \emph{without} a priori normalization or rescaling. Empirical results
demonstrate that \golf  can outperform state-of-the-art approaches; for example, we see more than $25$\% improvement in WAPE over other
methods on a public dataset that contains more than 100K-dimensional time series.
\end{abstract}

\section{Introduction}
\label{sec:intro}

Time-series forecasting is an important problem with many industrial
applications like retail demand forecasting~\cite{seeger2016bayesian},
financial predictions~\cite{kim2003financial}, predicting traffic or weather
patterns~\cite{chatfield2000time}. In general it plays a key role in
automating business processes~\cite{larson2001designing}. Modern data-sets can
have millions of correlated time-series over several thousand time-points. For
instance, in an online shopping portal like Amazon or Walmart, one may be interested
in the future daily demands for all items in a category, where the number of
items may be in millions. This leads to a problem of forecasting $n$
time-series (one for each of the $n$ items), given past demands over $t$
time-steps. Such a time series data-set can be represented as a matrix
$\*Y \in \-R^{n \times t}$ and we are interested in the
\emph{high-dimensional} setting where $n$ can be of the order of millions.

Traditional time-series forecasting methods operate on individual time-series
or a small number of time-series at a time. These methods include the well
known AR, ARIMA, exponential smoothing~\cite{mckenzie1984general}, the
classical Box-Jenkins methodology~\cite{box1968some} and more generally the
linear state-space models~\cite{hyndman2008forecasting}. However, these
methods are not easily scalable to large data-sets with millions of
time-series, owing to the need for individual training. Moreover, they cannot
benefit from shared temporal patterns in the whole data-set while training and
prediction.

Deep networks have gained popularity in time-series forecasting recently, due
to their ability to model non-linear temporal patterns. Recurrent Neural
Networks (RNN's)~\cite{funahashi1993approximation} have been popular in
sequential modeling, however they suffer from the gradient vanishing/exploding
problems in training. Long Short Term Memory (LSTM)~\cite{gers1999learning}
networks alleviate that issue and have had great success in langulage modeling
and other seq-to-seq tasks~\cite{gers1999learning,sundermeyer2012lstm}.
Recently, deep time-series models have used LSTM blocks as internal
components~\cite{flunkert2017deepar,rangapuram2018deep}. Another popular
architecture, that is competitive with LSTM's and arguably easier to train is
temporal convolutions/causal convolutions popularized by the wavenet
model~\cite{van2016wavenet}. Temporal convolutions have been recently used in
time-series forecasting~\cite{borovykh2017conditional,bai2018empirical}. These
deep network based models can be trained on large time-series data-sets as a
whole, in mini-batches. However, they still have two important shortcomings.

Firstly, most of the above deep models are difficult to train on data-sets
that have wide \textit{variation in scales} of the individual time-series. For
instance in the item demand forecasting use-case, the demand for some popular
items may be orders of magnitude more than those of niche items. In such
data-sets, each time-series needs to be appropriately normalized in order for
training to succeed, and then the predictions need to be scaled back to the
original scale. The mode and parameters of normalization are difficult to
choose and can lead to different accuracies. For example,
in~\cite{flunkert2017deepar,rangapuram2018deep} each time-series is whitened
using the corresponding empirical mean and standard deviation, while
in~\cite{borovykh2017conditional} the time-series are scaled by the
corresponding value on the first time-point.

Secondly, even though these deep models are trained on the entire data-set,
during prediction the models only focus on \textit{local} past data i.e only
the past data of a time-series is used for predicting the future of that
time-series. However, in most datasets, \textit{global} properties may be
useful during prediction time. For instance, in stock market predictions, it
might be beneficial to look at the past values of Alphabet, Amazon's stock
prices as well, while predicting the stock price of Apple. Similarly, in
retail demand forecasting, past values of similar items can be leveraged while
predicting the future for a certain item. To this end,
in~\cite{lai2018modeling}, the authors propose a combination of 2D convolution
and recurrent connections, that can take in multiple time-series in the input
layer thus capturing global properties during prediction. However, this method
does not scale beyond a few thousand time-series, owing to the growing size of
the input layer. On the other end of the spectrum, \trmf~\cite{yu2016temporal}
is a temporally regularized matrix factorization model that can express all
the time-series as linear combinations of \textit{basis time-series}. These
basis time-series can capture \textit{global} patterns during prediction.
However, \trmf can only model linear temporal dependencies. Moreover, there can
be approximation errors due to the factorization, which can be interpreted as
a lack of \text{local focus} i.e the model only concentrates on the global
patterns during prediction.

In light of the above discussion, we aim to propose a deep learning model that
can \emph{think globally and act locally} i.e., leverage both \textit{local and
global} patterns during training and prediction, and also can be trained
reliably even when there are \emph{wide variations in scale}. The {\bf main
contributions} of this paper are as follows:

\begin{compactitem}
\item  In Section~\ref{sec:leveled}, we discuss issues with wide variations in scale among different time-series, 
and propose a simple initialization scheme, \dln for Temporal Convolution Networks (TCN) that enables training without apriori normalization. 

\item In Section~\ref{sec:trmf},  we present a matrix factorization model
regularized by a TCN (\MFLN), that can express each time-series as
linear combination of $k$ basis time-series, where $k$ is much less than the
number of time-series. Unlike \trmf, this model can capture non-linear
dependencies as the regularization and prediction is done using a temporal convolution trained concurrently and also is amicable to scalable mini-batch
training. This model can handle \textit{global} dependencies during
prediction.

\item In Section~\ref{sec:hybrid}, we propose \golf, a \textit{hybrid} model, where the predictions from our \emph{global} \MFLN model, is provided as covariates 
for a temporal convolution network, thereby enabling the final model to focus both on local per time-series properties as well as global dataset wide properties, while both training and prediction. 

\item In Section~\ref{sec:results}, we show that \golf
outperforms other benchmarks on four real world time-series data-sets,
including a public \wiki dataset which contains more than $110K$ dimensions of time
series. More details can be found in Tables~\ref{tab:data} and~\ref{tab:results}.

\end{compactitem}

\section{Related Work}
\label{sec:rwork}
The literature on time-series forecasting is vast and spans several decades. Here, we will mostly focus on recent deep learning approaches. For a comprehensive treatment of traditional methods, we refer the readers to~\cite{hyndman2008forecasting,mckenzie1984general,box1968some,lutkepohl2005new,hamilton1994time} and the references there in.

In recent years deep learning models have gained popularity in time-series forecasting. DeepAR~\cite{flunkert2017deepar} proposes a LSTM based model where parameters of Bayesian models for the future time-steps are predicted as a function of the corresponding hidden states of the LSTM. In~\cite{rangapuram2018deep}, the authors combine linear state space models with deep networks. In~\cite{wen2017multi}, the authors propose a time-series model where all history of a time-series is encoded using an LSTM block, and a multi horizon MLP decoder is used to decode the input into future forecasts. LSTNet~\cite{lai2018modeling} can leverage correlations between multiple time-series through a combination of 2D convolution and recurrent structures. However, it is difficult to scale this model beyond a few thousand time-series because of the growing size of the input layer. Temporal convolutions have been recently used for time-series forecasting~\cite{borovykh2017conditional}.

Matrix factorization with temporal regularization was first used in~\cite{wilson2008regularized} in the context of speech de-noising. A spatio-temporal deep model for traffic data has been proposed in~\cite{yu2017spatio}. Perhaps closest to our work is TRMF~\cite{yu2016temporal}, where the authors propose an AR based temporal regularization. In this paper, we extend this work to non-linear settings where the temporal regularization can be performed by a temporal convolution network (see Section~\ref{sec:temporal}). We further combine the global matrix factorization model with a temporal convolution network, thus creating a hybrid model that can focus on both local and global properties. There has been a concurrent work~\cite{wang2019deep}, where an RNN has been used to evolve a global state common to all time-series. 
\section{Problem Setting}
\label{sec:pdef}

We consider the problem of forecasting high-dimensional time series over
future time-steps. High-dimensional time-series datasets consist of several
possibly correlated time-series evolving over time along with corresponding covariates, and the task is to forecast
the values of those time-series in future time-steps. Before, we formally
define the problem, we will set up some notation.

{\bf \noindent Notation: }  We will use bold capital letters to denote
matrices such as $\*M \in \-R^{m \times n}$. $M_{ij}$ and $\*M[i,j]$ will be
used interchangeably to denote the $(i,j)$-th entry of the matrix $\*M$. Let
$[n] \triangleq \{1,2,...,n\}$ for a positive integer $n$. For
$\cI \subseteq [m]$ and $\cJ \subseteq [n]$, the notation $\*M[\cI,\cJ]$ will
denote the sub-matrix of $\*M$ with rows in $\cI$ and columns in $\cJ$.
$\*M[:,\cJ]$ means that all the rows are selected and similarly $\*M[\cI,:]$
means all the columns are chosen. $\cJ+s$ denotes all the set of elements in $\cJ$ increased by $s$. The notation $i:j$ for positive integers
$j>i$, is used to denote the set $\{i,i+1,...,j\}$. $\norm{\*M}_F$,
$\norm{\*M}_2$ denote the Frobenius and Spectral norms respectively. 
We will also define $3$-dimensional tensor notation in a similar way as above. Tensors will also be represented by bold capital letters $\*T \in \-R^{m \times r \times n}$. $T_{ijk}$ and $\*T[i,j,k]$ will be
used interchangeably to denote the $(i,j,k)$-th entry of the tensor $\*T$.   For
$\cI \subseteq [m]$, $\cJ \subseteq [n]$ and $\_K \subseteq [r]$, the notation $\*T[\cI,\_K,\cJ]$ will
denote the sub-tensor of $\*T$, restricted to the selected coordinates.  
By convention, all vectors in this paper are row vectors unless otherwise
specified. $\norm{\*v}_p$ denotes the $p$-norm of the vector $\*v \in \-R^{1 \times n}$. $\*v_{\cI}$
denotes the sub-vector with entries $\{v_i : \forall i \in \cI\}$ where $v_i$
denotes the $i$-th coordinate of $\*v$ and $\cI \subseteq [n]$. The notation
$\*v_{i:j}$ will be used to denote the vector $[v_i,...,v_j]$. The notation
$[\*v;\*u] \in \-R^{1 \times 2n}$ will be used to denote the concatenation of
two row vectors $\*v$ and $\*u$. For a vector $\*v \in \-R^{1 \times n}$,
$\mu(\*v) \triangleq (\sum_i v_i) / n$ denotes the empirical mean of the
coordinates and
$\sigma(\*v) \triangleq \sqrt{(\sum_i (v_i - \mu(\*v)^2)) / n }$ denotes the
empirical standard deviation.

{\bf Forecasting Task: } A time-series data-set consists of the raw time-series, represented by a matrix
$\*Y = [\Ytr \Yte]$, where $\Ytr\in \-R^{n\times t}$,
$\Yte\in \-R^{n \times \tau}$, $n$ is the number of time-series, $t$
is the number time-points observed during training phase, $\tau$ is the window
size for forecasting. $\*y^{(i)}$ is used to denote the $i$-th
time series, i.e., the $i$-th row of $\*Y$. In addition to the raw time-series, there may optionally be observed covariates,  represented by the tensor $\*Z \in \-R^{n\times r\times (t+\tau)}$. 
$\*z^{(i)}_j = \*Z[i,:,j]$ denotes the $r$-dimensional covariates for time-series $i$ and time-point $j$. Here, the covariates can consist of global features like time of the day, day of the week etc which are common to all time-series, as well as covariates particular to each time-series, for example vectorized text features describing each time-series.  The forecasting task is to accurately predict the future in the test range, given the original time-series $\Ytr$  in the training time-range and the covariates $\*Z$.  $\predYte \in \-R^{n \times \tau}$ will be used to denote
the predicted values in the test range.

{\bf Objective: } The quality of the predictions are generally measured using
a metric calculated between the predicted and actual values in the test range. One
of the popular metrics is the normalized absolute deviation
metric~\cite{yu2016temporal}, defined as follows,
\begin{align}
\label{eq:metric}
\_L(Y^{\text{(obs)}}, Y^{\text{(pred)}}) =
\frac{\sum_{i = 1}^{n}\sum_{j=1}^{\tau} |Y^{\text{(obs)}}_{ij} -
Y^{\text{(pred)}}_{ij}|}{\sum_{i = 1}^{n}\sum_{j=1}^{\tau} |Y^{\text{(obs)}}_{ij}|},
\end{align}
where $Y^{\text{(obs)}}$ and $Y^{\text{(pred)}}$ are the observed and
predicted values, respectively. This
metric is also referred to as WAPE in the forecasting literature. Other
commonly used evaluation metrics are defined in Section~\ref{sec:metric}. Note
that \eqref{eq:metric} is also used as a loss function in one of our proposed models. We also use the squared-loss $\_L_2(Y^{\text{(obs)}}, Y^{\text{(pred)}}) = (1/n\tau ) \norm{Y^{\text{(obs)}}- Y^{\text{(pred)}}}_F^{2}$, during training. 

\section{\dln: Handling Diverse Scales with TCN}
\label{sec:temporal}
In this section, we propose $\dln$ a simple initialization scheme for Temporal Convolution Networks (TCN)~\cite{bai2018empirical} which is designed to handle
high-dimensional time-series data with wide variation in scale, without apriori
normalization. As mentioned before, deep networks are difficult to train on time-series datasets, where the individual time-series have diverse scales. LSTM based models cannot be reliably trained on such datasets and may need apriori standard normalization~\cite{lai2018modeling} or pre-scaling of the bayesian mean predictions~\cite{flunkert2017deepar}.  TCN's have also been shown to require apriori normalization~\cite{borovykh2017conditional} for time-series predictions. The choice of normalization parameters can have a significant effect on the prediction performance. Here, we show that a simple initialization scheme for the TCN network weights can alleviate this problem and lead to reliable training without apriori normalization. First, we will briefly discuss the TCN architecture. 

\begin{figure*}
	\centering
	\subfloat[][]{\includegraphics[width = 0.40\linewidth]{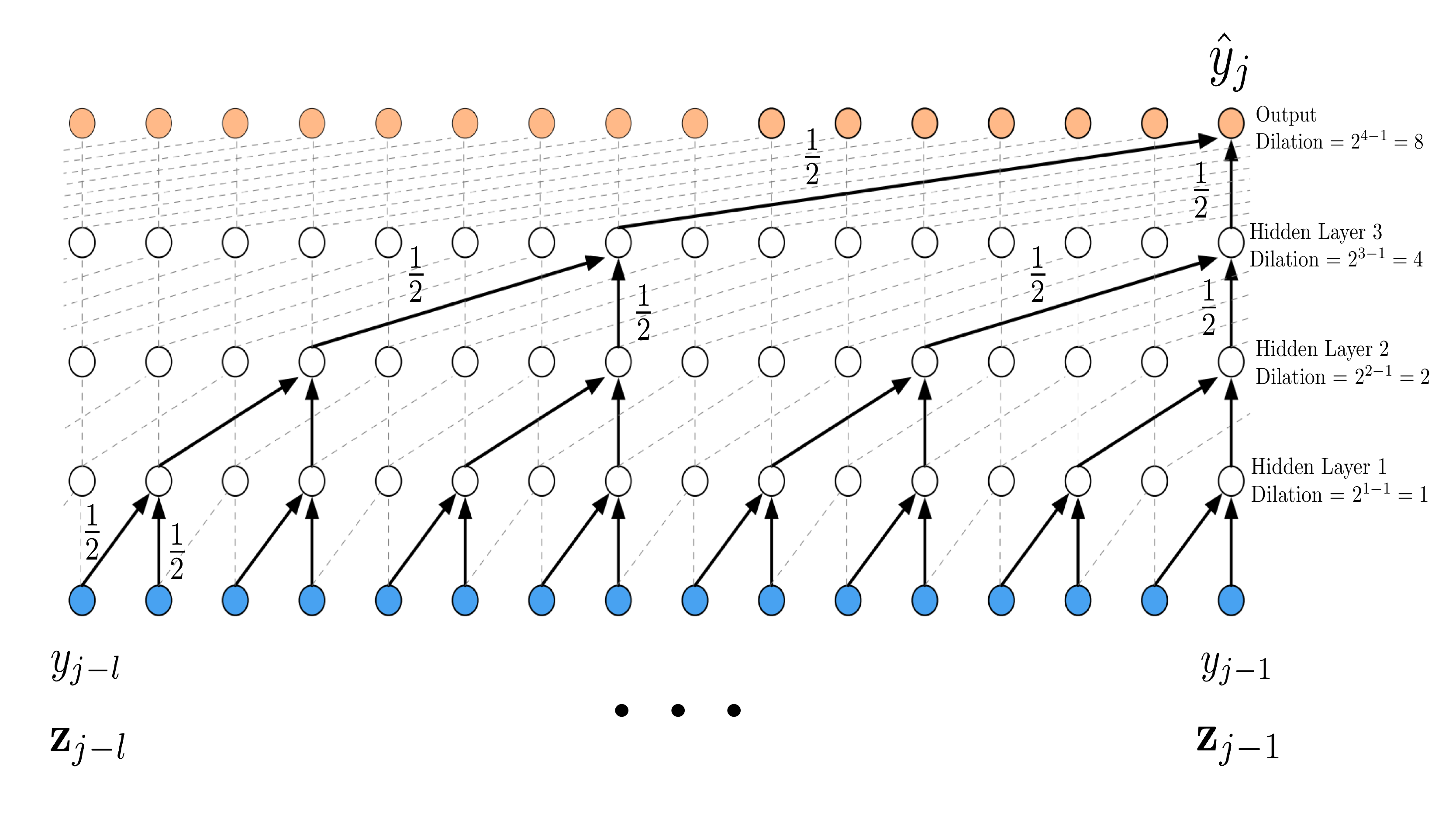}\label{fig:conv}} \hfill
	\subfloat[][]{\includegraphics[width = 0.40\linewidth]{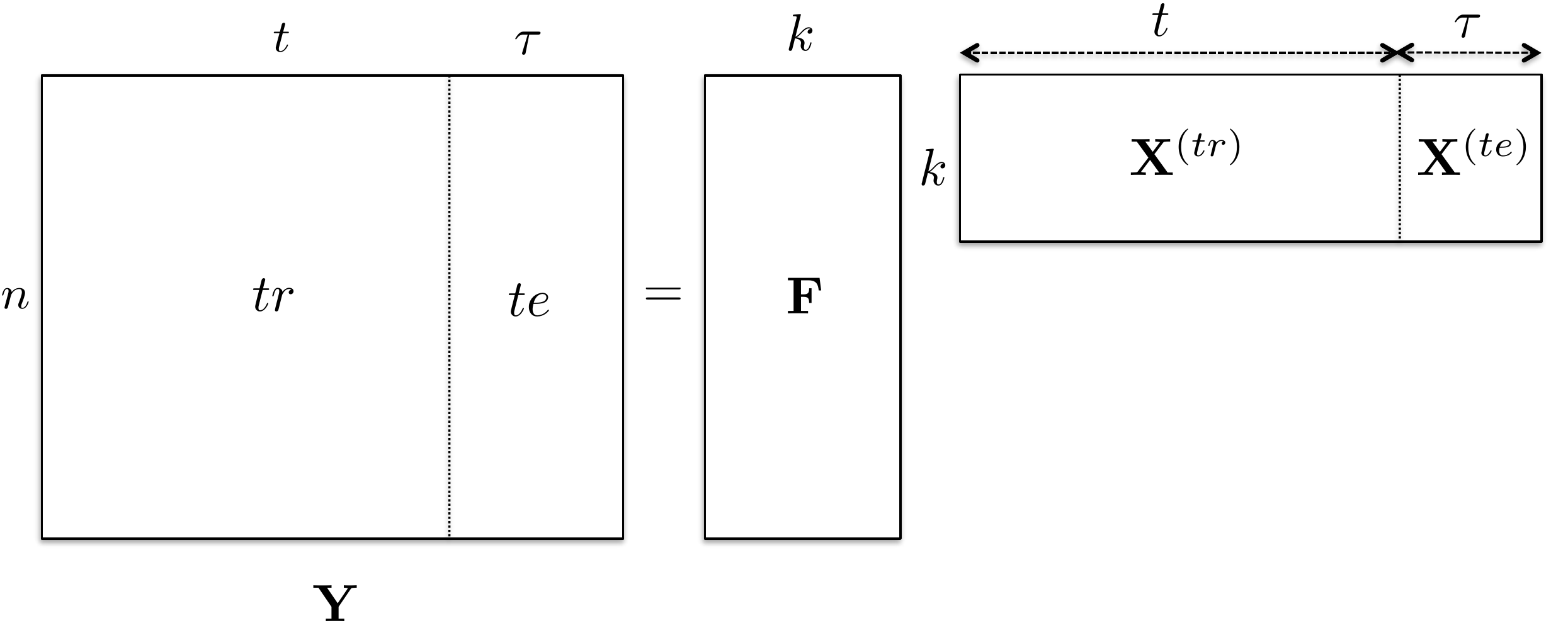}\label{fig:factors}} \hfill
	\caption{ \small Fig.~\ref{fig:conv} contains an illustration of a TCN. Note that the base image has been borrowed from~\cite{van2016wavenet}. The network has $d=4$ layers, with filter size $k=2$. The network maps the input $\*y_{t-l:t-1}$ to the one-shifted output $\hat{\*y}_{t-l+1:t}$. 
		Figure~\ref{fig:factors} presents an illustration of the matrix factorization approach in
		time-series forecasting. The $\Ytr$ training matrix can be factorized
		into low-rank factors $\*F$ ($\in \mathbb{R}^{n \times k}$) and $\Xtr$ (
		$\in \mathbb{R} ^{k \times t}$). If $\Xtr$ preserves temporal
		structures then the future values $\Xte$ can be predicted by a
		time-series forecasting model and thus the test period predictions can be made
		as $\*F\Xte$. }
	\label{fig:allthree}
	\vspace{-1em}
\end{figure*}

{\bf Temporal Convolution: } Temporal convolution (also referred to as Causal
Convolution)~\cite{van2016wavenet,borovykh2017conditional,bai2018empirical}
are multi-layered neural networks comprised of 1D convolutional layers with dilation. The dilation in layer $i$ is usually set as
$\mathrm{dil}(i) = 2^{i-1}$. A temporal convolution network with filter size $k$ and
number of layers $d$ has a dynamic range (or look-back) of
$l' := 1 + l = 1 + 2(k-1)2^{d-1}$. Note that it is assumed that the stride is $1$ in every layer and layer $i$ needs to be zero-padded in the beginning with $(k-1)\mathrm{dil}(i)$ zeros. An example of a temporal convolution network with one
channel per layer is shown in Fig.~\ref{fig:conv}.  For more details, we
refer the readers to the general description in~\cite{bai2018empirical}. Note
that in practice, we can have multiple channels per layer of a TC network.  The TC network can thus be treated as an object that takes
in the previous values of a time-series $\*y_{\cJ}$, where
$\cJ=\cbr{j - l, j - l + 1, \cdots, j - 1}$ along with the past covariates $\mathbf{z}_{\cJ}$, corresponding to that time-series and outputs the
one-step look ahead predicted value $\hat{\*y}_{\cJ+1}$. We will denote
a temporal convolution neural network by $\cT(\cdot|\Theta)$, where $\Theta$
is the parameter weights in the temporal convolution network. Thus, we have
$\hat{\*y}_{\cJ+1} = \cT(\*y_{\cJ},\*z_{\cJ} | \Theta)$. The same operators can be defined
on matrices. Given $\*Y \in \-R^{n \times t}, \*Z \in \-R^{n \times r \times (t+\tau)}$ and a set of row indices
$\cI = \{i_1,...,i_{b_n}\}\subset [n]$, we can write
$\hat{\*Y}[\cI,\cJ + 1] = \cT(\*Y[\cI,\cJ], \*Z[\cI,:,\cJ]|\Theta)$.

{\bf \dln Scheme: } One possible method to alleviate the issues with diverse scales, is to start with initial network parameters, that results in approximately predicting the average value of a given window of past time-points $\*y_{j-l:j-1}$, as the future prediction $\hat{y}_{j}$. The hope is that, over the course of training, the network would learn to predict variations around this mean prediction, given that the variations around this mean is relatively scale free.  This can be achieved through a simple initialization scheme for some configurations of TCN, which we call \dln. For ease of exposition, let us consider the setting without covariates and only one channel per layer, which can be functionally represented as $\hat{\*y}_{\cJ+1} = \cT(\*y_{\cJ}| \Theta)$.  In this case, the initialization scheme is to simply set all the filter weights to $1/k$, where $k$ is the filter size in every layer. This results in a proposition.

\begin{proposition}[\dln]
	\label{prop:leveled}
Let $\_T(\cdot| \Theta)$ be a TCN with one channel per layer, filter size $k = 2$, number of layers $d$. Here $\Theta$ denotes the weights and biases of the network. Let $[\hat{y}_{j-l+1}, \cdots, \hat{y}_{j}]  := \cT(\*y_{\cJ}| \Theta)$ for $\cJ = \{j-l,....,j-1\}$ and $ l = 2(k-1)2^{d-1}$.  If all the biases in $\Theta$ are set to $0$ and all the weights set to $1/k$ then, $  \hat{y}_{j} =  \mu (\*y_{\cJ})$ if $\*y \geq \*0$ and all activation functions are ReLUs. 
\end{proposition}

The above proposition shows that \dln results in a prediction $\hat{y}_{j}$, which is the average of the past $l$ time-points, where $l$ is the dynamic range of the TCN, when filter size is $k = 2$ (see Fig.~\ref{fig:conv}). The proof of proposition~\ref{prop:leveled} (see Section~\ref{sec:leveled})) follows from the fact that an activation value in an internal layer is the average of the corresponding $k$ inputs from the previous layer, and an induction on the layers yields the results.  \dln can also be extended to the case with covariates, by setting the corresponding filter weights to $0$ in the input layer. It can also be easily extended to multiple channels per layer with $k=2$. In Section~\ref{sec:results}, we show that a TCN with \dln can be trained reliably without apriori normalization, on real world datasets, even for values of $k \neq 2$. We provide the psuedo-code for training a TCN with \dln as Algorithm~\ref{algo:TCtrain}. 

Note that we have also experimented with a more sophisticated variation of the Temporal Convolution architecture termed as Deep Leveled Network (\dlnprev), which we include in Appendix~\ref{sec:leveled}. However, we observed that the simple initialization scheme for TCN (\dln) can match the performance of the Deep Leveled network.

\section{\golf: A Deep Global Local Forecaster}
\label{sec:deepglo}
In this section we will introduce our hybrid model \golf, that can leverage both global and local features, during training and prediction. In Section~\ref{sec:trmf}, we present the
global component, TCN regularized Matrix Factorization  (\MFLN). This model can capture global patterns
during prediction, by representing each of the original time-series as a
linear combination of $k$ basis time-series, where $k \ll n$. In
Section~\ref{sec:hybrid}, we detail how the output of the global model can be incorporated as an input covariate dimension for a TCN, thus leading to a hybrid model that can both focus on local per time-series signals and leverage global dataset wide components.

\begin{minipage}{0.43\textwidth}
	\vspace{-1em}
	\begin{algorithm}[H]
		\caption{Mini-batch Training for TCN with \dln}
		\label{algo:TCtrain}
		\scriptsize
		\begin{algorithmic}[1]
			\Require {learning rate $\eta$, horizontal batch size $b_t$, vertical
				batch size $b_n$, and $\texttt{maxiters}$}
			\State Initialize $\_T(\cdot|\Theta)$ according to \dln
			\For {$\texttt{iter} = 1, \cdots, \texttt{maxiters}$}
			\For {each batch with indices $\cI$ and $\cJ$ in an epoch}
			\State $\cI = \cbr{i_1,...,i_{b_n}}$ and $\cJ = \cbr{j+1,j+2,...,j + b_t}$
			\State
			$\hat{\*Y} \leftarrow \cT(\*Y[\cI,\cJ], \*Z[\cI,:,\cJ]|\Theta)$ 
			\State $\Theta \leftarrow \Theta -\eta \frac{\partial}{\Theta} \_L(\*Y[\cI,\cJ+1], \hat{\*Y})$
			\EndFor
			\EndFor
		\end{algorithmic}
	\end{algorithm}
\end{minipage}
\hfill
\begin{minipage}{0.43\textwidth}
		\begin{algorithm}[H]
		\caption{Temporal Matrix Factorization Regularized by TCN (\MFLN)}
		\label{algo:dtrmf}
		\scriptsize
		\begin{algorithmic}[1]
			\Require $\texttt{iters}_{\texttt{init}}$, $\texttt{iters}_{\texttt{train}}$, $\texttt{iters}_{\texttt{alt}}$.
			\State {/* {\em Model Initialization} */}
			\State Initialize $\cT_X(\cdot) $ by \dln 
			\State Initialize $\*F$ and $\Xtr$ by Alg~\ref{algo:factors} for $\texttt{iters}_{\texttt{init}}$ iterations.
			\State {/* {\em Alternate training cycles} */}
			\For {$\texttt{iter} = 1,2,...,\texttt{iters}_{\texttt{alt}}$}
			\State Update $\*F$ and $\Xtr$ by Alg~\ref{algo:factors} for $\texttt{iters}_{\texttt{train}}$ iterations
			\State Update $\cT_X(\cdot)$ by Alg~\ref{algo:TCtrain} on $\Xtr$ for $\texttt{iters}_{\texttt{train}}$ iterations, with \textit{no covariates}. 
			\EndFor
		\end{algorithmic}
	\end{algorithm}
	
	\vspace{-1em}
\end{minipage}

\subsection{Global: Temporal Convolution Network regularized Matrix Factorization (\MFLN)}
\label{sec:trmf}
In this section we propose a low-rank matrix factorization model for
time-series forecasting that uses a TCN for regularization. 
The idea is to factorize the training time-series matrix $\Ytr$ into low-rank
factors $\*F\in \-R^{n \times k}$ and $\Xtr \in \-R^{k \times t}$,
where $k \ll n$. This is illustrated in Figure~\ref{fig:factors}. Further, we
would like to encourage a temporal structure in $\Xtr$ matrix, such that the
future values $\Xte$ in the test range can also be forecasted. Let
$\*X = [\Xtr \Xte]$. Thus, the matrix $\*X$ can be thought of to be
comprised of $k$ \textit{basis time-series} that capture the global temporal
patterns in the whole data-set and all the original time-series are linear
combinations of these basis time-series. In the next subsection we will
describe how a TCN can be used to encourage the temporal structure for $\*X$.

{\bf Temporal Regularization by a TCN:} If we are provided with a TCN that captures the temporal patterns in the training data-set
$\Ytr$, then we can encourage temporal structures in $\Xtr$ using
this model. Let us assume that the said network is $\_T_X(\cdot)$. The
temporal patterns can be encouraged by including the following temporal
regularization into the objective function:
\begin{align}
&\_R(\Xtr\mid \_T_X(\cdot)) := \frac{1}{|\cJ|}\_L_2\rbr{\*X[:,\cJ],
	\cT_{X}\rbr{\*X[:, \cJ - 1]}},\label{eq:temporal_loss}
\end{align}
where $\cJ = \cbr{2, \cdots, t}$ and $\_L_2(\cdot,\cdot)$ is the squared-loss metric, defined before. This implies that the values of the $\*X^{(tr)}$ on
time-index $j$ are close to the predictions from the temporal network applied
on the past values between time-steps $\{j-l,...,j-1\}$. Here, $l+1$ is 
equal to the dynamic range of the network defined in
Section~\ref{sec:temporal}. Thus the overall loss function for the factors and
the temporal network is as follows:
\begin{align}
\cL_G(\Ytr,\*F,\Xtr,\cT_X)
:= \cL_2\rbr{\Ytr, \*F\Xtr} + \lambda_{\cT} \cR(\Xtr \mid \cT_X(\cdot)), \label{eq:global_loss}
\end{align}
where $\lambda_{\_T}$ is the regularization parameter for the temporal
regularization component.

{\bf Training:} The low-rank factors $\*F,\Xtr$ and the temporal network
$\cT_X(\cdot)$ can be trained alternatively to approximately minimize the loss
in Eq.~\eqref{eq:global_loss}. The overall training can be performed through
mini-batch SGD and can be broken down into two main components performed
alternatingly: $(i)$ given a fixed $\cT_X(\cdot)$ minimize
$\cL_G(\*F,\Xtr,\cT_X)$ with respect to the factors $\*F,\Xtr$ -
Algorithm~\ref{algo:factors} and $(ii)$ train the network $\cT_X(\cdot)$ on
the matrix $\Xtr$ using Algorithm~\ref{algo:TCtrain}.

The overall algorithm is detailed in Algorithm~\ref{algo:dtrmf}. The 
TCN $\cT_X(\cdot)$ is first initialized by \dln. Then in
the second initialization step, two factors $\*F$ and $\Xtr$
are trained using the initialized $\cT_X(\cdot)$ (step 3), for
$\texttt{iters}_{\texttt{init}}$ iterations.
This is followed by the $\texttt{iters}_{\texttt{alt}}$ alternative steps to
update $\*F$, $\Xtr$, and $\cT_X(\cdot)$ (steps 5-7). 

{\bf Prediction: } The trained model local network $\_T_X(\cdot)$ can be used for multi-step
look-ahead prediction in a standard manner. Given the past data-points of a basis
time-series, $\*x_{j-l:j-1}$, the prediction for the next time-step, 
$\hat{x}_j$ is given by  $[\hat{x}_{j-l+1}, \cdots, \hat{x}_{j}]  := \cT_X(\*x_{j-l:j-1})$
Now, the one-step look-ahead prediction can be concatenated with the past
values to form the sequence
$\tilde{\*x}_{j-l+1:j} = [\*x_{j-l+1:j-1}  \hat{x}_{j}]$, which can be again
passed through the network to get the next prediction:
$ [\cdots, \hat{x}_{j+1}] = \cT_X(\tilde{\*x}_{j-l+1:j})$. The same procedure can be repeated $\tau$ times to predict $\tau$ time-steps
ahead in the future. Thus we can obtain the basis time-series in the test-range $\hat{\*X}^{(te)}$. The final global predictions are then given by $\*Y^{(te)} = \*F\hat{\*X}^{(te)}$. 

\begin{remark}
	Note that \MFLN can perform rolling predictions without retraining unlike \trmf. We provide more details in Appendix~\ref{sec:roll}, in the interest of space.
\end{remark}

\begin{minipage}{0.43\textwidth}
	
		\begin{algorithm}[H]
		\caption{ Training the Low-rank factors $\*F,\Xtr$ given a fixed network $\_T_X(\cdot)$, for one epoch }
		\label{algo:factors}
		\scriptsize
		\begin{algorithmic}[1]
			\Require learning rate $\eta$, a TCN $\cT_X(\cdot)$.
			\For {each batch with indices $\cI$ and $\cJ$ in an epoch}
			\State $\cI = \cbr{i_1,...,i_{b_n}}$ and $\cJ = \cbr{j+1,j+2,...,j + b_t}$
			\State $\*X[:,\cJ] \leftarrow \*X[:,\cJ]  -\eta \frac{\partial}{\partial\*X[:,\cJ]} \cL_G(\*Y[\cI,\cJ],\*F[\cI,:], \*X[:, \cJ], \cT_X)$
			\State $\*F[\cI,:] \leftarrow \*F[\cI,:] - \eta \frac{\partial}{\partial\*F[\cI,:]} \cL_G(\*Y[\cI,\cJ],\*F[\cI,:], \*X[:, \cJ], \cT_X)$
			\EndFor
		\end{algorithmic}
	\end{algorithm}

\end{minipage}
\hfill
\begin{minipage}{0.43\textwidth}

	\begin{algorithm}[H]
		\caption{\golf - Deep Global Local Forecaster}
		\label{algo:complete}
		\scriptsize
		\begin{algorithmic}[1]
			\State Obtain global $\*F$, $\Xtr$ and $\cT_X(\cdot)$ by Alg~\ref{algo:dtrmf}.
			\State Initialize $\_T_Y(\cdot)$ with number of inputs $r+2$ and \dln. 
			\State  {/* {\em Training hybrid model} */ }
			\State  Let $\hat{\*Y}^{(g)}$ be the global model prediction in the training range.
			\State  Create covariates $\*Z' \in \mathbb{R}^{n \times (r + 1)  \times t}$  s.t $\*Z'[:,1,:] = \hat{\*Y}^{(g)}$ and $\*Z'[:,2:r+1,:] = \*Z[:,:,1:t]$.
			\State  Train  $\_T_Y(\cdot)$ using Algorithm~\ref{algo:TCtrain} with time-series $\*Y^{(tr)}$ and covariates $\*Z'$. 
		\end{algorithmic}
	\end{algorithm}
	
\end{minipage}

\subsection{Combining the Global Model with Local Features}
\label{sec:hybrid}
In this section, we present our final hybrid model. Recall that our forecasting task has as input the training raw time-series $\*Y^{(tr)}$ and the covariates $\*Z \in \mathbb{R}^{n\times r \times (t+\tau)}$. Our hybrid forecaster is a TCN $\_T_Y(\cdot | \Theta_Y)$ with a input size of $r+2$ dimensions: $(i)$ one of the inputs is reserved for the past points of the original raw time-series, $(ii)$ $r$ inputs for the original $r$-dimensional covariates and $(iii)$ the remaining dimension is reserved for the output of the global \MFLN model, which is fed as input covariates. The overall model is demonstrated in Figure~\ref{fig:illu}. The training pseudo-code for our model is provided as Algorithm~\ref{algo:complete}. 

\begin{figure*}
	\vspace{-1em}
	\centering
	
	\subfloat[][]{\includegraphics[width = 0.40\linewidth]{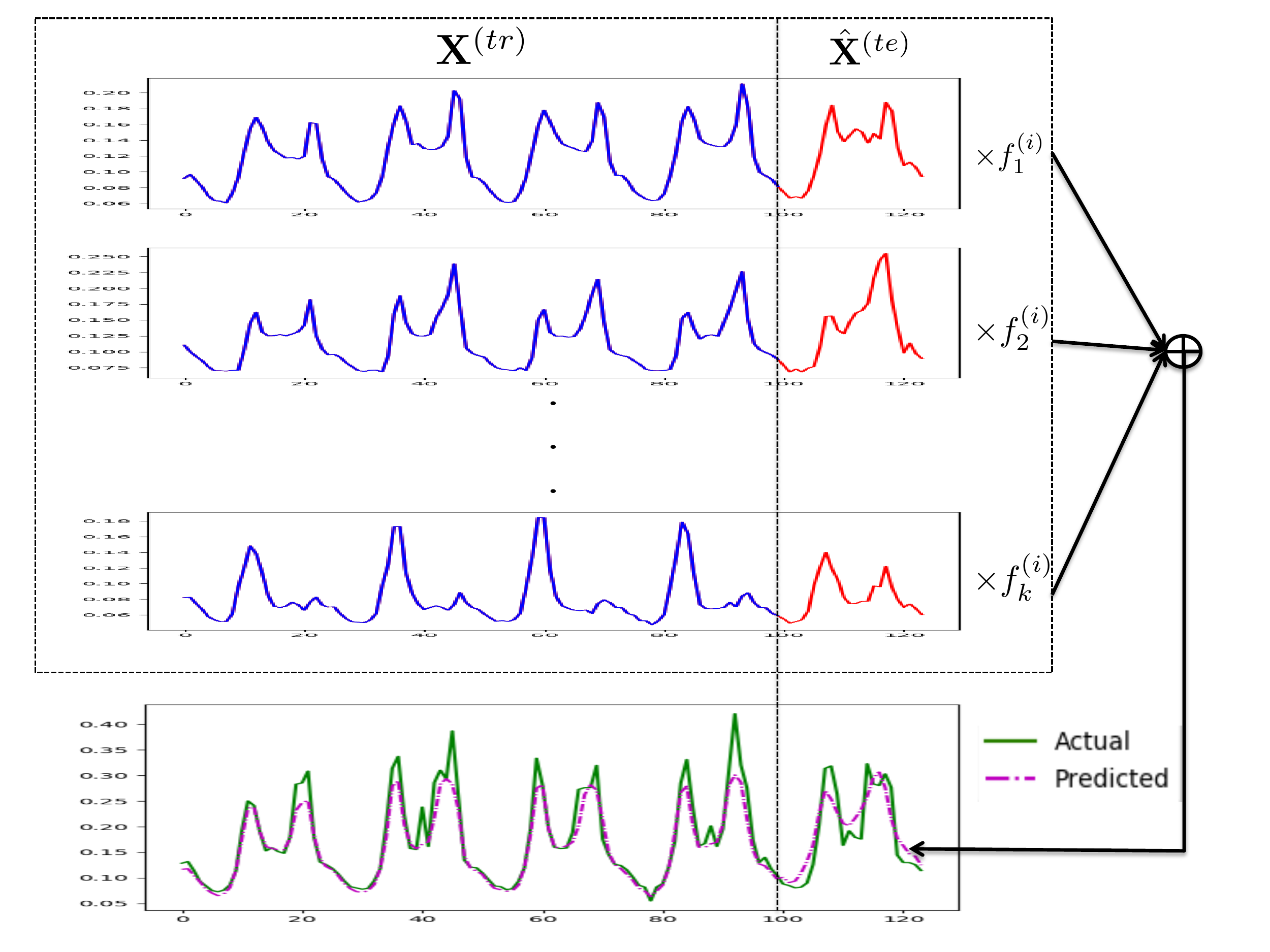}\label{fig:xnet}} \hfill
	\subfloat[][]{\includegraphics[width = 0.40\linewidth]{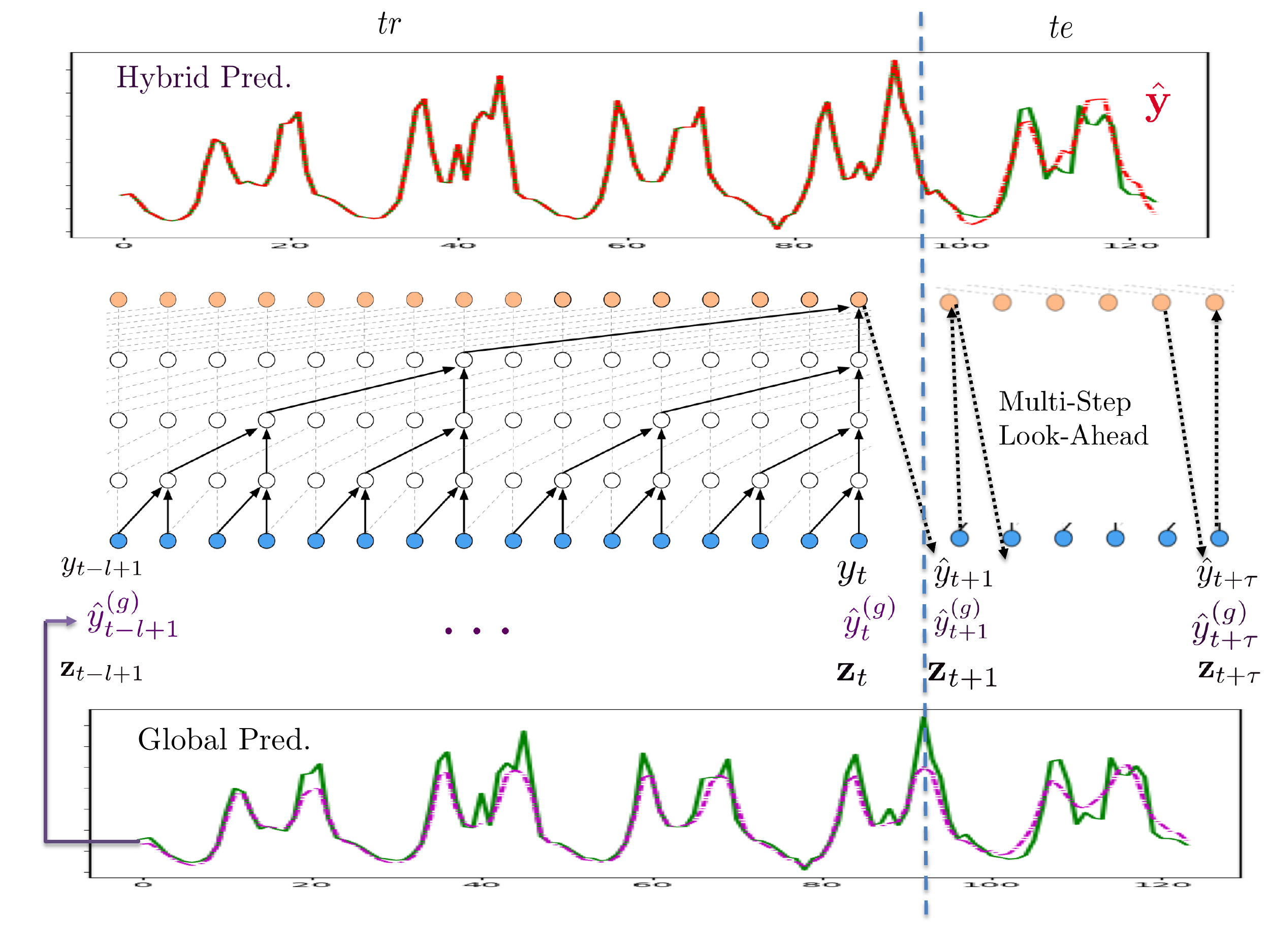}\label{fig:glo}}
	\caption{ \small In Fig.~\ref{fig:xnet}, we show some of the basis
		time-series extracted from the \traffic dataset, which can be combined
		linearly to yield individual original time-series. It can be seen that the
		basis series are highly temporal and can be predicted in the test range using
		the network $\cT_X(\cdot|\Theta_X)$. In Fig.~\ref{fig:glo} (base image borrowed from~\cite{van2016wavenet}), we show an illustration of \golf. The TCN shown is the network $\_T_Y(\cdot)$, which takes in as input the original time-points, the original covariates and the output of the global model as covariates. Thus this network can combine the local properties with the output of the global model during prediction. }
	\label{fig:illu}
	\vspace{-1em}
\end{figure*}

Therefore, by providing the global \MFLN model prediction as covariates to a TCN, we can make the final predictions a function of global dataset wide properties as well as the past values of the local time-series and covariates. Note that both rolling predictions and multi-step look-ahead predictions can be performed by \golf, as the global \MFLN model and the hybrid TCN
$\cT_Y(\cdot)$ can perform the forecast, without any need for re-training. We illustrate some representative results on a time-series from
the \ dataset in Fig.~\ref{fig:illu}. In
Fig.~\ref{fig:xnet}, we show some of the basis time-series (global features)
extracted from the \traffic dataset, which can be combined linearly to yield
individual original time-series. It can be seen that the basis series are
highly temporal and can be predicted in the test range using the network
$\cT_X(\cdot|\Theta_X)$. In Fig.~\ref{fig:glo}, we illustrate the complete architecture of \golf. It can be observed that the output of the global \MFLN model is inputted as a covariate to the TCN $\_T_Y(\cdot)$, which inturn combines this with local features, and predicts in the test range through multi-step look-ahead predictions. 

\begin{table}
	 \vspace{-.5em}
	\caption{\small Data statistics and Evaluation settings. In
  the rolling Pred., $\tau_w$ denotes the number of time-points in each
window and $n_w$ denotes the number of rolling windows.
$\mathrm{std}(\{\mu\})$ denotes the standard deviation among the means of all
the time series in the data-set i.e
$\mathrm{std}(\{\mu\}) =  \mathrm{std}(\{\mu(\*y^{(i)})\}_{i=1}^{n})$.
Similarly, $\mathrm{std}(\{\mathrm{std}\})$ denotes the standard deviation
among the std. of all the time series in the data-set i.e
$\mathrm{std}(\{\mathrm{std}\}) =  \mathrm{std}(\{\mathrm{std}(\*y^{(i)})\}_{i=1}^{n})$. It can be seen that the \elec and \wiki datasets have wide variations in
scale. }
	\label{tab:data}
  \begin{center}
	\small
    \begin{tabular}{l|rrrrcc}
      Data      & $n$    & $t$    &$\tau_w$ & $n_w$ & $\mathrm{std}(\{\mu(\*y_i)\})$ & $\mathrm{std}(\{\mathrm{std}(\*y_i)\})$	\\
      \hline
      \hline
      \elec       & 370    & 25,968 & 24 & 7 & $1.19e4$ & $7.99e3$ \\
      \hline
      \traffic    & 963    & 10,392 & 24 & 7 & $1.08e{-2}$  & $1.25e{-2}$ \\
      \hline
      \wiki       & 115,084& 747    & 14 & 4 & $4.85e{4} $& $1.26e4$ \\
      \hline
      \pems & 228 & 11,232 & 9 & 160 & $3.97$  & $4.42$ \\
      \hline
    \end{tabular}
	
  \end{center}
\vspace{-1em}
\end{table}

\begin{table*}[!ht]
	\caption{\small Comparison of algorithms on normalized and unnormalized versions of
		data-sets on rolling prediction tasks The error metrics reported are
		WAPE/MAPE/SMAPE (see Section~\ref{sec:metric}). \trmf is retrained before
		every prediction window, during the rolling predictions. All other models are
		trained once on the initial training set and used for further prediction for
		all the rolling windows. Note that for \deepar, the normalized column represents model trained with \texttt{scaler=True} and unnormalized represents \texttt{scaler=False}. Prophet could not be scaled to the \wiki dataset, even
		though it was parallelized on a 32 core machine. Below the main table, we provide MAE/MAPE/RMSE comparison with the models implemented in~\cite{yu2016temporal}, on the \pems dataset.}
	\label{tab:results}
	\small
	\begin{center}
		\resizebox{\textwidth}{!}{
			\begin{tabular}{@{}c|c|@{ }cc@{ }||@{ }cc@{ }||@{ }cc@{}}
				& \multirow{2}{*}{Algorithm} &
				\multicolumn{2}{c}{\elec $n=370$} &
				\multicolumn{2}{c}{\traffic $n=963$} &
				\multicolumn{2}{c}{\wiki $n=115,084$} \\
				& & Normalized & Unnormalized & Normalized & Unnormalized & Normalized & Unnormalized \\
				\hline
				\hline
				\multirow{3}{*}{Proposed} & \golf & 0.133/0.453/0.162& {\bf 0.082}/0.341/{\bf 0.121} & 0.166/ 0.210 /0.179 &  {\bf 0.148}/{\bf 0.168}/{\bf 0.142} & 0.569/3.335/1.036 & 0.237/0.441/0.395\\

				& Local TCN (\dln)& 0.143/0.356/0.207& 0.092/{\bf 0.237}/0.126 & 0.157/{\bf 0.201}/0.156 & 0.169/0.177/0.169 & {\bf 0.243}/{\bf 0.545}/{\bf 0.431} & \bf{ 0.212}/{\bf 0.316}/{\bf0.296} \\
				
				& Global \MFLN & 0.144/0.485/0.174& 0.106/0.525/0.188 & 0.336/0.415/0.451 & 0.226/0.284/0.247 & 1.19/8.46/1.56 & 0.433/1.59/0.686 \\
				\hline
				\multirow{4}{*}{Local-Only} & LSTM & 0.109/0.264/0.154 & 0.896/0.672/0.768  & 0.276/0.389/0.361 & 0.270/0.357/0.263 & 0.427/2.170/0.590 & 0.789/0.686/0.493 \\

				& \deepar & {\bf 0.086}/{\bf 0.259}/ \bf{0.141} & 0.994/0.818/1.85 & {\bf0.140}/{\bf 0.201}/{ \bf 0.114} & 0.211/0.331/0.267 & 0.429/2.980/0.424 & 0.993/8.120/1.475 \\

				& TCN (no \dln). & 0.147/0.476/0.156 & 0.423/0.769/0.523 & 0.204/0.284/0.236  & 0.239/0.425/0.281 & 0.336/1.322/0.497 & 0.511/0.884/0.509 \\

				& Prophet & 0.197/0.393/0.221 & 0.221/0.586/0.524  & 0.313/0.600/0.420 & 0.303/0.559/0.403 & - & - \\
				\hline
				\multirow{2}{*}{Global-Only} & \trmf(retrained) & 0.104/0.280/0.151 &  0.105/0.431/0.183 &  0.159/0.226/ 0.181 &  0.210/ 0.322/ 0.275 & 0.309/0.847/0.451 & 0.320/0.938/0.503 \\

				& SVD+TCN & 0.219/0.437/0.238 & 0.368/0.779/0.346  & 0.468/0.841/0.580 & 0.329/0.687/0.340 & 0.697/3.51/0.886 & 0.639/2.000/0.893 \\
				\hline
			\end{tabular}
			\vspace{1em}
		} \hfill

		\begin{tabular}{c|@{ }cc@{ }}
			\small
       Algorithm & \pems (MAE/MAPE/RMSE) \\
      \hline
      \hline
			 \golf (Unnormalized) & {\bf 3.53}/ {\bf 0.079} / {\bf 6.49}\\
			 \golf (Normalized) & 4.53/ 0.103 / 6.91 \\
			 STGCN(Cheb) & 3.57/0.087/6.77 \\
			 STGCN($1^{st}$) & 3.79/0.091/7.03 \\
			\hline
		\end{tabular}
	\end{center}
	\vspace{-.5em}
\end{table*}

\section{Empirical Results}
\label{sec:results}

In this section, we validate our model on four real-world data-sets on
rolling prediction tasks (see Section~\ref{sec:rolling} for more details)
against other benchmarks. The data-sets in consideration are, $(i)$
\elec~\cite{elec} - Hourly load on $370$ houses. The training set consists of
$25,968$ time-points and the task is to perform rolling validation for the next
7 days (24 time-points at a time, for 7 windows) as done
in~\cite{yu2016temporal,rangapuram2018deep,flunkert2017deepar};$(ii)$
\traffic~\cite{cuturi2011fast} - Hourly traffic on $963$ roads in San
Francisco. The training set consists of $10m392$ time-points and the task is to
perform rolling validation for the next 7 days (24 time-points at a time, for
7 windows) as done
in~\cite{yu2016temporal,rangapuram2018deep,flunkert2017deepar} and $(iii)$
\wiki~\cite{wiki} - Daily web-traffic on about $115,084$ articles from Wikipedia. We
only keep the time-series without missing values from the original data-set.
The values for each day are normalized by the total traffic on that day across
all time-series and then multiplied by $1e8$. The training set consists of
$747$ time-points and the task is to perform rolling validation for the next
86 days, 14 days at a time. More data statistics indicating scale variations
are provided in Table~\ref{tab:data}.  $(iv)$ \pems~\cite{chen2001freeway} - Data collected from Caltrain PeMS system, which contains data for $228$ time-series,
collected at 5 min interval. The training set consists of
$11232$ time-points and we perform rolling validation for the next
1440 points, 9 points at a time.


For each data-set, all models are compared on two different settings. In the
first setting the models are trained on \textit{normalized} version of the
data-set where each time series in the training set is re-scaled as
$\tilde{\*y}^{(i)}_{1:t-\tau} = (\*y^{(i)}_{1:t-\tau} - \mu(\*y^{(i)}_{1:t-\tau}))/( \sigma(\*y^{(i)}_{1:t-\tau}))$
and then the predictions are scaled back to the original scaling. In the
second setting, the data-set is \textit{unnormalized} i.e left as it is while
training and prediction. Note that all models are compared in the test range
on the original scale of the data. The purpose of these two settings is to
measure the impact of scaling on the accuracy of the different models.

\vspace{-0.5em}
The models under contention are:
{\bf (1)} \golf: The combined local and global  model proposed in Section~\ref{sec:hybrid}. We use time-features like time-of-day, day-of-week etc. as global covariates, similar to \deepar. For a more detailed discussion, please refer to Section~\ref{sec:mparams}.
{\bf (2)} Local TCN (\dln): The temporal convolution based architecture discussed in Section~\ref{sec:temporal}, with \dln.
{\bf (3)} LSTM: A simple LSTM block that predicts the time-series values as function of the hidden states~\cite{gers1999learning}.
{\bf (4)} \deepar: The model proposed in~\cite{flunkert2017deepar}.
{\bf (5)} TCN: A simple Temporal Convolution model as described in Section~\ref{sec:temporal}.
{\bf (6)} Prophet: The versatile forecasting model from Facebook based on classical techniques~\cite{prophet}.
{\bf (7)} \trmf: the model proposed in ~\cite{yu2016temporal}. Note that this model needs to be retrained for every rolling prediction window.
{\bf (8)} SVD+TCN: Combination of SVD and TCN. The original data is factorized as $\*Y=\*U\*V$ using SVD and a leveled network is trained on the $\*V$. This is a simple baseline for a global-only approach.
 {\bf (9)} STGCN: The spatio-temporal models in~\cite{yu2017spatio}.
 We use the same hyper-parameters for \deepar on the \traffic and
\elec datasets, as specified in~\cite{flunkert2017deepar}, as implemented in GluonTS~\cite{gluonts}. The WAPE
values from the original paper could not be directly used, as there are
different values reported in~\cite{flunkert2017deepar}
and~\cite{rangapuram2018deep}. Note that for \deepar, normalized and unnormalized settings corresponds to using \texttt{sclaing=True} and \texttt{scaling=False} in the GluonTS package.
The model in \trmf~\cite{yu2016temporal} was
trained with different hyper-parameters (larger rank) than in the original
paper and therefore the results are slightly better. More details about all
the hyper-parameters used are provided in Section~\ref{sec:mresults}. The rank
$k$ used in \elec, \traffic, \wiki and \pems are $64/60,64/60$, $256/1,024$ and $64/-$ for \golf/\trmf.

\vspace{-0.5em}
In Table~\ref{tab:results}, we report WAPE/MAPE/SMAPE (see definitions in
Section~\ref{sec:metric}) on the first three datasets under both normalized and
unnormalized training. We can see
that \golf features among the top two models in almost all categories, under all three
metrics. \golf does better than the individual local TCN (\dln) method and the global \MFLN model on average, as it is aided by both global and local factors.  The local TCN (\dln) model  performs the best on the larger \wiki dataset with $>30\%$  improvement over all models (not proposed in this paper) in terms of WAPE, while DeepGLO is close behind with greater than $25\%$ improvement over all other models. We also find that \golf performs better in the unnormalized setting on all instances, because there is no need for scaling the input and rescaling back the outputs of the model.  We find that the TCN (no \dln), \deepar\footnote{Note that for \deepar, normalized means the GluonTS implementation run with \texttt{scaler=True} and unnormalized means \texttt{scaler=False}} and the LSTM models do not perform well in the unnormalized setting as expected. In the second table, we compare \golf with the models in~\cite{yu2017spatio}, which can capture global features but require a weighted graph representing closeness relations between the time-series as \textit{external} input. We see that \golf (unnormalized) performs the best on all metrics, even though it does not require any external input.  Our implementation can be found at https://github.com/rajatsen91/deepglo.


\bibliography{dtrmf}
\bibliographystyle{plain}

\clearpage
\appendix

\section{Deep Leveled Network}
\label{sec:leveled}
Large scale time-series datasets containing upwards of hundreds of thousands
of time-series can have very diverse scales. The diversity in scale leads to
issues in training deep models, both Temporal Convolutions and LSTM based
architectures, and some normalization is needed for training to
succeed~\cite{lai2018modeling,borovykh2017conditional,flunkert2017deepar}.
However, selecting the correct normalizing factor for each time-series is not
an exact science and can have effects on predictive performance. For instance
in~\cite{flunkert2017deepar} the data-sets are whitened using the training
standard deviation and mean of each time-series while training, and the
predictions are renormalized. On the other hand
in~\cite{borovykh2017conditional}, each time-series is rescaled by the value
of that time-series on the first time-step. Moreover, when performing rolling
predictions using a pre-trained model, when new data is observed there is a
potential need for updating the scaling factors by incorporating the new
time-points. In this section we propose \dlnprev, a simple \textit{leveling network}
architecture that can be trained on diverse datasets without the need for a
priori normalization.

\dlnprev consists of two temporal convolution blocks (having the same dynamic
range/look-back $l$) that are trained concurrently. Let
us denote the two networks and the associated weights by $\cT_m(\cdot|\Theta_m)$
and $\cT_r(\cdot|\Theta_r)$ respectively. The key idea is to have $\cT_m(\cdot|\Theta_m)$
(the leveling component) to predict the rolling mean of the next $w$ future
time-points given the past. On the other-hand $\cT_r(\cdot|\Theta_r)$ (the residual
component) will be used to predict the variations with respect to this mean
value. Given an appropriate window size $w$ the
rolling mean stays stable for each time-series and can be predicted by a
simple temporal convolution model and given these predictions the additive
variations are relatively scale free i.e. the network $\cT_r(\cdot|\Theta_r)$ can
be trained reliably without normalization. This can be summarized by the
following equations:
\begin{align}
[\hat{y}_{j-l+1}, \cdots, \hat{y}_{j}] &= \cT_{\dlnprev}\rbr{\*y_{j-l:j-1}|\Theta_m, \Theta_r} 
:= \cT_m\rbr{\*y_{j-l:j-1}|\Theta_m}\! +\! \cT_r(\*y_{j-l:j-1}|\Theta_r) \label{eq:leveled}  \\
[\hat{m}_{j-l+1}, \cdots, \hat{m}_{j}] &= \cT_m(\*y_{j-l:j-1}|\Theta_m)\label{eq:leveled-mean}\\
[\hat{r}_{j-l+1}, \cdots, \hat{r}_{j}] &= \cT_r(\*y_{j-l:j-1}|\Theta_r)\label{eq:leveled-residual}
\end{align}
where we want $\hat{m}_{j}$ to be close to $\mu(\*y_{j:j+w-1})$ and
$\hat{r}_{j}$ to be close to $y_j - \mu(\*y_{j:j+w-1})$. An illustration of
the leveled network methodology is shown in the above figure.

{\bf \noindent Training:} Both the networks can be trained concurrently given
the training set $\Ytr$, using mini-batch stochastic gradient updates.
The pseudo-code for training a \dlnprev is described in
Algorithm~\ref{algo:TCtrain}. The loss
function $\_L(\cdot,\cdot)$ used is the same as the metric defined in
Eq.~\eqref{eq:metric}. Note that in Step 9, the leveling component
$\cT_m(\cdot|\Theta_m)$ is held fixed and only $\cT_r(\cdot|\Theta_r)$ is updated.



\begin{figure}
	\centering
	\includegraphics[width = 0.45\linewidth]{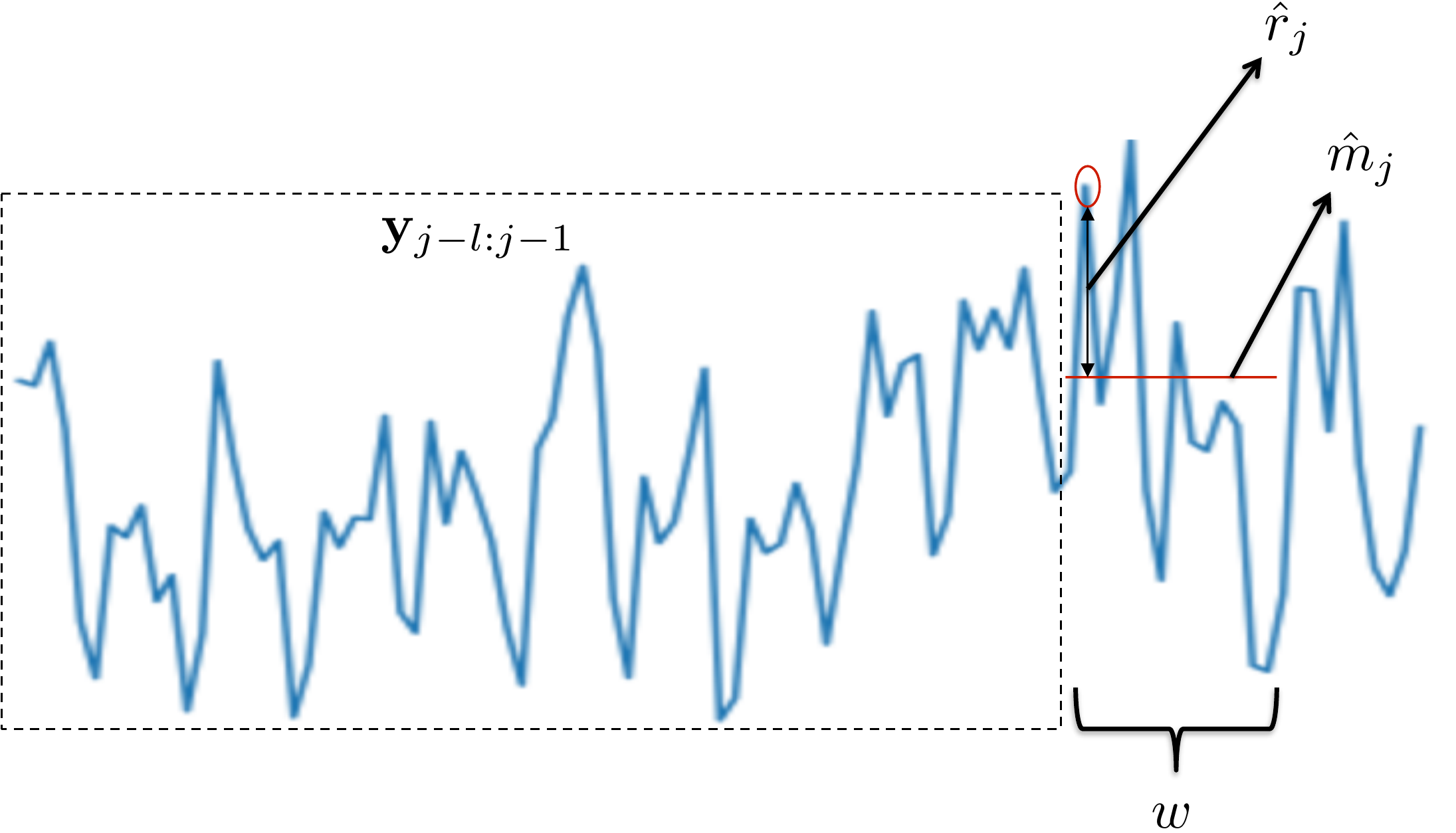}
	\label{fig:level}
	\caption{Illustration of \dlnprev}
\end{figure}

{\bf \noindent Prediction:} The trained model can be used for multi-step
look-ahead prediction in a standard manner. Given the past data-points of a
time-series, $\*y_{j-l:j-1}$, the prediction for the next time-step is given
by $\hat{y}_j$ defined in \eqref{eq:leveled}.
Now, the one-step look-ahead prediction can be concatenated with the past
values to form the sequence
$\tilde{\*y}_{j-l+1:j} = [\*y_{j-l+1:j-1}  \hat{y}_{j}]$, which can be again
passed through the network and get the next prediction:
$ [\cdots, \hat{y}_{j+1}] = \cT_{\text{DLN}}(\tilde{\*y}_{j-l+1:j})$. The same procedure can be repeated $\tau$ times to predict $\tau$ time-steps
ahead in the future. Table~\ref{tab:levres} shows the performance of leveled network on the same datasets and mertics used in Table~\ref{tab:results}. 

\begin{table*}[!ht]
	\caption{\small Performance of \dlnprev on the same datasets and metrics as in Table~\ref{tab:results}. }
	\label{tab:levres}
	\small
	\begin{center}
		\resizebox{\textwidth}{!}{
			\begin{tabular}{@{}c|c|@{ }cc@{ }||@{ }cc@{ }||@{ }cc@{}}
				& \multirow{2}{*}{Algorithm} &
				\multicolumn{2}{c}{\elec $n=370$} &
				\multicolumn{2}{c}{\traffic $n=963$} &
				\multicolumn{2}{c}{\wiki $n=115,084$} \\
				& & Normalized & Unnormalized & Normalized & Unnormalized & Normalized & Unnormalized \\
				\hline
				\hline
				\multirow{1}{*}{Proposed} & Local \dlnprev & 0.086/0.258/0.129& 0.118/ 0.336/0.172 & 0.169/0.246/0.218 & 0.237/0.422/0.275 & 0.235/0.469/ 0.346 & 0.288/0.397/0.341 \\
				\hline
			\end{tabular}
			\vspace{1em}
		}
	\end{center}
\end{table*}

We will now provide a proof of Proposition~\ref{prop:leveled}. 

\begin{proof}[Proof of Proposition~\ref{prop:leveled}] 
	We will follow the convention of numbering the input layer as layer $0$ and each subsequent layers in increasing order. ~\ref{fig:conv} shows a TCN with the last layered numbered $d=4$. Each neuron in a layer is numbered starting at $0$ from the right-hand side. $a_{i,j}$ denote neuron $j$ in layer $i$. We will focus on the neurons in each layer that take part in the prediction $\hat{y}_j$.  Note that on layer $d-1$ ($d$ being the last layer), the neurons that are connected to the output $\hat{y}_j$, are $a_{d-1,0}$ and $a_{d-1,2^{d-1}}$. Therefore, we have $\hat{y}_j = \frac{1}{2} (a_{d-1,0} + a_{d-1,2^{d-1}})$, whenever the neurons have values greater equal to zero.  Similary, any neuron $a_{l,s*2^{l}}$ is the average of $a_{l-1,s*2^{l}}$ and $a_{l-1,(2s+1)*2^{l-1}}$.  Therefore, by induction we have that,
	
	\begin{equation}
	\hat{y}_j = \frac{1}{2*2^{d-1}} \sum_{i=1}^{2*2^{d-1}} y_{j-i}. 
	\end{equation}
	
\end{proof}

\section{Rolling Prediction without Retraining}
\label{sec:roll} 
Once the \MFLN model is trained using
Algorithm~\ref{algo:dtrmf}, we can make predictions on the test range using
multi-step look-ahead prediction. The method is straight-forward - we first
use $\cT_X(\cdot)$ to make multi-step look-ahead prediction on the basis
time-series in $\Xtr$ as detailed in Section~\ref{sec:temporal}, to
obtained $\predXte$; then the original time-series predictions can be
obtained by $\predYte = \*F \predXte$. This model can also be
adapted to make rolling predictions \textit{without retraining}. In the case
of rolling predictions, the task is to train the model on a training period
say $\*Y[:,1:t_1]$, then make predictions on a future time period say
$\hat{\*Y}[:,t_1+1:t_2]$, then receive the actual values on the future
time-range $\*Y[:,t_1+1:t_2]$ and after incorporating these values make
further predictions for a time range further in the future
$\hat{\*Y}[:,t_2+1:t_3]$ and so on. The key challenge in this scenario, is to
incorporate the newly observed values $\*Y[:,t_1+1:t_2]$ to generate the
values of the basis time-series in that period which is $\*X[:,t_1+1:t_2]$.
We propose to obtain these values by minimizing global loss defined in
\eqref{eq:global_loss} while keeping $\*F$ and $\cT_X(\cdot)$ fixed:
\begin{align*}
\*X[:,t_1+1:t_2] 
= \argmin_{\*M \in \-R^{k \times (t_2-t_1)}} \cL_{G}\rbr{\*Y[:,
	t_1 + 1 : t_2], \*F, \*M, \cT_X}.
\end{align*}
Once we obtain $\*X[:,t_1+1:t_2]$, we can make
predictions in the next set of future time-periods $\hat{\*Y}[:,t_2+1:t_3]$.
Note that the \trmf model in~\cite{yu2016temporal} needed to be retrained from
scratch to incorporate the newly observed values. In this work retraining is
not required to achieve good performance, as we shall see in our experiments
in Section~\ref{sec:results}.

\section{More Experimental Details}
\label{sec:mresults}
We will provide more details about the experiments like the exact rolling prediction setting in each data-sets, the evaluation metrics and the model hyper-parameters.

\subsection{Rolling Prediction} 
\label{sec:rolling}
In our experiments in Section~\ref{sec:results}, we compare model performances on rolling prediction tasks. The goal in this setting is to predict future time-steps in batches as more data is revealed. Suppose the initial training time-period is $\{1,2,...,t_0\}$, rolling window size $\tau$ and number of test windows $n_w$. Let $t_i = t_0 + i\tau$. The rolling prediction task is a sequential process, where given data till last window, $\*Y[:,1:t_{i-1}]$, we predict the values for the next future window $\hat{\*Y}[:,t_{i-1}+1:t_i]$, and then the actual  values for the next window $\*Y[:,t_{i-1}+1:t_i]$ are revealed and the process is carried on for $i = 1,2,...,n_w$. The final measure of performance is the loss $\_L(\hat{\*Y}[:,t_{0}+1:t_{n_w}],\*Y[:,t_{0}+1:t_{n_w}])$ for the metric $\_L$ defined in Eq.~\eqref{eq:metric}. 

For instance, In the traffic data-set experiments we have $t_0 = 10392, \tau = 24, w = 7$ and in electricity $t_0 = 25968, \tau = 24, w = 7$. The wiki data-set experiments have the parameters $t_0 = 747, \tau = 14, w = 4$.

\subsection{Loss Metrics} 
\label{sec:metric}
The following well-known loss metrics~\cite{hyndman2008forecasting} are used in this paper. Here, $\*Y \in \-R^{n'\times t'}$ represents the actual values while  $\hat{\*Y} \in \-R^{n'\times t'}$ are the corresponding predictions.

{\bf (i) WAPE: } Weighted Absolute Percent Error is defined as follows,
\begin{align}
\label{eq:wape}
\_L(\hat{\*Y},\*Y) = \frac{\sum_{i = 1}^{n'}\sum_{j=1}^{t'} |Y_{ij} - \hat{Y}_{ij}|}{\sum_{i = 1}^{n'}\sum_{j=1}^{t'} |Y_{ij} |}.
\end{align}

{\bf (ii) MAPE: } Mean Absolute Percent Error is defined as follows,
\begin{align}
\label{eq:mape}
\_L_m(\hat{\*Y},\*Y) = \frac{1}{Z_0}\sum_{i = 1}^{n'}\sum_{j=1}^{t'} \frac{|Y_{ij} - \hat{Y}_{ij}|}{ |Y_{ij} |} \mathds{1}\{|Y_{ij}| > 0\},
\end{align}
where $Z_0 = \sum_{i = 1}^{n'}\sum_{j=1}^{t'} \mathds{1}\{|Y_{ij}| > 0\}$.

{\bf (iii) SMAPE: } Symmetric Mean Absolute Percent Error is defined as follows,
\begin{align}
\label{eq:smape}
\_L_s(\hat{\*Y},\*Y) = \frac{1}{Z_0}\sum_{i = 1}^{n'}\sum_{j=1}^{t'} \frac{2|Y_{ij} - \hat{Y}_{ij}|}{ |Y_{ij} + \hat{Y}_{ij} |} \mathds{1}\{|Y_{ij}| > 0\},
\end{align}
where $Z_0 = \sum_{i = 1}^{n'}\sum_{j=1}^{t'} \mathds{1}\{|Y_{ij}| > 0\}$.

\subsection{Model Parameters and Settings} 
\label{sec:mparams}
In this section we will describe the compared models in more details. For a TC network the important parameters are the kernel size/filter size, number of layers and number of filters/channels per layer. A network described by $[c_1,c_2,c_3]$ implies that there are three layers with $c_i$ filters in layer $i$. For, an LSTM the parameters $(n_h,n_l)$ means that the number of neurons in hidden layers is $n_h$, and number of hidden layers is $n_l$. All models are trained with early stopping with a tenacity or patience of $7$, with a maximum number of epochs $300$. The hyper-parameters for all the models are as follows,

{\bf DeepGLO: } In all the datasets, except \wiki and \pems the networks $\_T_X$ and $\_T_Y$ both have parameters $[32,32,32,32,32,1]$ and kernel size is $7$. On the \wiki dataset, we set the networks  $\_T_X$ and $\_T_Y$ with parameters $[32,32,32,32,1]$. On the \pems dataset we set the parameters as $[32,32,32,32,32,1]$ and $[16,16,16,16,16,1]$. We set $\alpha$ and $\lambda_{\_T}$ both to $0.2$ in all experiments. The rank $k$ used in \elec, \traffic, \wiki and \pems are $64,64$ ,$256$ and $64$ respectively. We use $7$ time-covariates, which includes minute of the hour, hour of the day, day of the week, day of the month, day of the year, month of the year, week of the year, all normalized in a range $[-0.5,0.5]$, which is a subset of the time-covariates used by default in the GluonTS library. 

{\bf Local TCN (\dln): } We use the setting $[32,32,32,32,32,1]$ for all datasets. 

{\bf Local DLN: }  In all the datasets, the leveled networks  have parameters $[32,32,32,32,32,1]$ and kernel size is $7$.

{\bf TRMF: } The rank $k$ used in electricity, traffic, wiki are $60,60$ ,$1024$ respectively. The lag indices are set to include the last day and the same day in the last week for traffic and electricity data. 

{\bf SVD+Leveled: } The rank $k$ used in electricity, traffic and wiki are $60,60$ and $500$ respectively. In all the datasets, the leveled networks  have parameters $[32,32,32,32,32,1]$ and kernel size is $7$.

{\bf LSTM: } In all datasets the parameters are $(45,3)$. 

{\bf DeepAR: } In all datasets, we use the default parameters in the \texttt{DeepAREstimator} in the GluonTS implementation. 

{\bf TCN: } In all the datasets, the parameters are $[32,32,32,32,32,1]$ and kernel size is $7$.

{\bf Prophet: } The parameters are selected automatically. The model is parallelized over $32$ cores. The model was run with growth = 'logistic', as this was found to perform the best. 

Note that we replicate the exact values reported in~\cite{yu2017spatio} for the STGCN models on the \pems dataset.

\end{document}